\documentclass[conference]{IEEEtran}
\IEEEoverridecommandlockouts
\usepackage{algorithm}
\usepackage{graphicx}
\usepackage[noend]{algpseudocode}
\usepackage{url}
\usepackage{breqn}
\usepackage{cite}
\usepackage[usenames]{color}
\usepackage{amsfonts}
\usepackage{fancyhdr}
\usepackage{todonotes}
\usepackage{arydshln}
\usepackage{comment}
\usepackage{amsmath,amssymb,amsthm}
\usepackage{psfrag,caption,subcaption}
\newtheorem*{theorem}{Theorem}
\newtheorem*{proposition}{Proposition}
\newtheorem{lemma}{Lemma}

\usepackage{arydshln}
\def\cast{{
   \mathord{
      \hbox to 0em{
         \ooalign{
	   \smash{\hbox{$\ast$}}\crcr
	   \smash{\hskip-1pt\Large\hbox{$\circ$}} }
	 \hidewidth}
      \phantom{\bigcirc}
} }}
\newcommand{\bds}{\begin {itemize}}
\newcommand{\eds}{\end {itemize}}
\newcommand{\bdf}{\begin{definition}}
\newcommand{\blm}{\begin{lemma}}
\newcommand{\edf}{\end{definition}}
\newcommand{\elm}{\end{lemma}}
\newcommand{\bthm}{\begin{theorem}}
\newcommand{\ethm}{\end{theorem}}
\newcommand{\bprp}{\begin{prop}}
\newcommand{\eprp}{\end{prop}}
\newcommand{\bcl}{\begin{claim}}
\newcommand{\ecl}{\end{claim}}
\newcommand{\bcr}{\begin{coro}}
\newcommand{\ecr}{\end{coro}}
\newcommand{\bquest}{\begin{question}}
\newcommand{\equest}{\end{question}}

\newcommand{\larrow}{{\larrow}}

\newcommand{\argmin}{\ensuremath{\mathrm{arg}\min}}
\newcommand{\argmax}{\ensuremath{\mathrm{arg}\max}}

\newcommand{\cS}{{\ensuremath{\mathcal{S}}}}

\newcommand{\cV}{{\ensuremath{\mathcal{V}}}}

\newcommand{\cX}{{\ensuremath{\mathcal{X}}}}

\newcommand{\vc}{{\ensuremath{{\mathbf{c}}}}}

\newcommand{\ve}{{\ensuremath{{\mathbf{e}}}}}

\newcommand{\mA}{{\ensuremath{\mathbf{A}}}}

\newcommand{\mB}{{\ensuremath{\mathbf{B}}}}

\newcommand{\mD}{{\ensuremath{\mathbf{D}}}}

\newcommand{\mL}{{\ensuremath{\mathbf{L}}}}

\newcommand{\mM}{{\ensuremath{\mathbf{M}}}}

\usepackage{latexsym}

\def\IC{\mathbb C}
\def\IN{\mathbb N}
\def\IZ{\mathbb Z}
\def\IR{\mathbb R}

\def\shat{^{\mathchoice{}{}%
 {\,\,\smash{\hbox{\lower4pt\hbox{$\widehat{\null}$}}}}%
 {\,\smash{\hbox{\lower3pt\hbox{$\hat{\null}$}}}}}}

\def\bSigma{{
      \ooalign{
      \smash{\hskip.4pt\raise.4pt\hbox{$\Sigma$}}\vphantom{}\crcr
      \smash{\hskip.7pt\raise.6pt\hbox{$\Sigma$}}\vphantom{}\crcr
      \smash{\hbox{$\Sigma$}}\vphantom{$\Sigma$}}
      \vphantom{\hbox{$\Sigma$}}
      }}
\def\bTheta{{
      \ooalign{
      \smash{\hskip.5pt\raise.5pt\hbox{$\Theta$}}\vphantom{}\crcr
      \smash{\hskip.0pt\raise.1pt\hbox{$\Theta$}}\vphantom{}\crcr
      \smash{\hbox{$\Theta$}}\vphantom{$\Theta$}}
      \vphantom{\hbox{$\Theta$}}
      }}
\def\bDelta{{
      \ooalign{
      \smash{\hskip.4pt\raise.4pt\hbox{$\Delta$}}\vphantom{}\crcr
      \smash{\hskip.7pt\raise.6pt\hbox{$\Delta$}}\vphantom{}\crcr
      \smash{\hbox{$\Delta$}}\vphantom{$\Delta$}}
      \vphantom{\hbox{$\Delta$}}
      }}
\def\bLambda{{
      \ooalign{
      \smash{\hskip.5pt\raise.5pt\hbox{$\Lambda$}}\vphantom{}\crcr
      \smash{\hskip.0pt\raise.1pt\hbox{$\Lambda$}}\vphantom{}\crcr
      \smash{\hbox{$\Lambda$}}\vphantom{$\Lambda$}}
      \vphantom{\hbox{$\Lambda$}}
      }}

\makeatletter

\def\bordermatrix#1{\begingroup \m@th
  \@tempdima 8.75\p@
  \setbox\z@\vbox{%
    \def\cr{\crcr\noalign{\kern2\p@\global\let\cr\endline}}%
    \ialign{$##$\hfil\kern2\p@\kern\@tempdima&\thinspace\hfil$##$\hfil
      &&\quad\hfil$##$\hfil\crcr
      \omit\strut\hfil\crcr\noalign{\kern-\baselineskip}%
      #1\crcr\omit\strut\cr}}%
  \setbox\tw@\vbox{\unvcopy\z@\global\setbox\@ne\lastbox}%
  \setbox\tw@\hbox{\unhbox\@ne\unskip\global\setbox\@ne\lastbox}%
  \setbox\tw@\hbox{$\kern\wd\@ne\kern-\@tempdima\left[\kern-\wd\@ne
    \global\setbox\@ne\vbox{\box\@ne\kern2\p@}%
    \vcenter{\kern-\ht\@ne\unvbox\z@\kern-\baselineskip}\,\right]$}%
  \null\;\vbox{\kern\ht\@ne\box\tw@}\endgroup}
\makeatother

\makeatletter
\def\argmin{\mathop{\operator@font arg\,min}}
\def\argmax{\mathop{\operator@font arg\,max}}
\makeatother

\newcommand{\bea}{\begin{array}}
\newcommand{\ena}{\end{array}}
\newcommand{\beq}{\begin{equation}}
\newcommand{\enq}{\end{equation}}

\newcommand{\beqa}{\begin{eqnarray}}
\newcommand{\enqa}{\end{eqnarray}}

\newcommand{\beqan}{\begin{eqnarray*}}
\newcommand{\enqan}{\end{eqnarray*}}

\newcommand{\AL}{\begin{enumerate}}
\newcommand{\ALE}{\end{enumerate}}

\def\addots{\mathinner{
    \mkern1mu\raise0pt\vbox{\kern7pt\hbox{.}}
    \mkern2mu\raise4pt\hbox{.}
    \mkern2mu\raise7pt\hbox{.}
    \mkern1mu}}

\def\sddots{\mathinner{
    \mkern.8mu\raise7pt\hbox{.}
    \mkern.8mu\raise4pt\hbox{.}
    \mkern.8mu\raise0pt\vbox{\kern7pt\hbox{.}}
    \mkern1mu}}

\def\saddots{\mathinner{
    \mkern.2mu\raise0pt\vbox{\kern7pt\hbox{.}}
    \mkern.2mu\raise4pt\hbox{.}
    \mkern.2mu\raise7pt\hbox{.}
    \mkern1mu}}

\def\sqplus{\mathbin{
	{\ooalign{\hfil\raise.3ex\hbox{\scriptsize
	+}\hfil\crcr\mathhexbox274\crcr\mathhexbox275}}
	}} 
\def\sqminus{\mathbin{
	{\ooalign{\hfil\raise.3ex\hbox{\scriptsize
	--}\hfil\crcr\mathhexbox274\crcr\mathhexbox275}}
	}}

\def\IC{{
   \mathord{
      \hbox to 0em{
	 \hskip-4pt
         \ooalign{
	   \smash{\hskip1.9pt\raise2.6pt\hbox{$\scriptscriptstyle |$}}\crcr
	   \smash{\hbox{\rm\sf C}} }
	 \hidewidth}
      \phantom{\hbox{\rm\sf C}}
} }}
\def\IN{
    {\ooalign{
   \smash{\hskip2.2pt\raise1.5pt\hbox{$\scriptscriptstyle |$}}\vphantom{}\crcr
   \hbox{\sf N}
	}}
	} 
\def\IZ{
    {\ooalign{
   \smash{\hskip1.9pt\raise0pt\hbox{$\sf Z$}}\vphantom{}\crcr
   \hbox{\sf Z}
	}}
	} 
\def\IR{
    {\ooalign{
   \smash{\hskip2.2pt\raise1.5pt\hbox{$\scriptscriptstyle |$}}\vphantom{}\crcr
   \smash{\hskip2.2pt\raise3.3pt\hbox{$\scriptscriptstyle |$}}\vphantom{}\crcr
   \hbox{\sf R}
	}}
	} 

\DeclareMathAlphabet{\mathcmb}{OT1}{cmr}{b}{n}

\def\bSigma{\ensuremath{\mathcmb{\Sigma}}}
\def\bLambda{\ensuremath{\mathcmb{\Lambda}}}

\def\bTheta{\ensuremath{\mathcmb{\Theta}}}

\newcommand{\SI}{\begin{indlist}}
\newcommand{\EI}{\end{indlist}}

\newcommand{\DL}{\begin{dashlist}}
\newcommand{\DLE}{\end{dashlist}}

\makeatletter
\def\setboxz@h{\setbox\z@\hbox}
\def\wdz@{\wd\z@}
\def\boxz@{\box\z@}
\def\underset#1#2{\binrel@{#2}%
  \binrel@@{\mathop{\kern\z@#2}\limits_{#1}}}
\def\binrel@#1{\begingroup
  \setboxz@h{\thinmuskip0mu
    \medmuskip\m@ne mu\thickmuskip\@ne mu
    \setbox\tw@\hbox{$#1\m@th$}\kern-\wd\tw@
    ${}#1{}\m@th$}%
  \edef\@tempa{\endgroup\let\noexpand\binrel@@
    \ifdim\wdz@<\z@ \mathbin
    \else\ifdim\wdz@>\z@ \mathrel
    \else \relax\fi\fi}%
  \@tempa
}
\let\binrel@@\relax%
\makeatother

\begin{document}
%\title{Clustering Simplicial Complex Using Second Order Simplices 
\title{%Simplicial Spectral Clustering
Clustering with Simplicial Complexes
%\thanks{Identify applicable funding agency here. If none, delete this.}
}
\author{
\IEEEauthorblockN{Thummaluru Siddartha Reddy$^\star$, Sundeep Prabhakar Chepuri$^\star$, and Pierre Borgnat$^\dagger$\\
\emph{$^\star$Indian Institute of Science, Bengaluru, India}\\
\emph{$^\dagger$ENS de Lyon, CNRS, Lab. Physique, Lyon, France}
}

}
\maketitle
\begin{abstract}
In this work, we propose a new clustering algorithm to group nodes in networks based on second-order simplices (aka filled triangles) to leverage higher-order network interactions. We define a simplicial conductance function, which on minimizing, yields an optimal partition with a higher density of filled triangles within the set while the density of filled triangles is smaller across the sets. To this end, we propose a simplicial adjacency operator that captures the relation between the nodes through  second-order simplices. This allows us to extend the well-known Cheeger inequality to cluster a simplicial complex. Then, leveraging the Cheeger inequality, we propose the simplicial spectral clustering algorithm. We report results from numerical experiments on synthetic and real-world network data to demonstrate the efficacy of the proposed approach. 
\end{abstract}

\section{Introduction} \label{sec:introduction}

Networks model complex interactions (as edges) between entities (as nodes). Networks often have community structures, and determining these communities is a topic of significant interest in network science
\cite{Fortunato2010,benson2016higher}. Clustering algorithms detect communities by partitioning the nodes in a network into sets with high edge density within a set while maintaining a low edge density between nodes of different sets. Such partitioning is promoted by cut criteria, such as modularity~\cite{newman2006modularity} or edge conductance~\cite{lu2018community,shi2000normalized}.
Spectral clustering~\cite{ng2001spectral,von2007tutorial} is a well-known heuristic obtained by relaxing the edge conductance cut criterion and relating it to the second smallest eigenvalue of the graph Laplacian through the so-called Cheeger inequality. However, these methods only capture the pairwise relations between nodes. In many real-world networks~\cite{newman2001structure,granovetter1973strength}, 
 \emph{supra-pairwise} relations, i.e., beyond pairwise relationships, e.g., triadic or more, are common~\cite{BATTISTON20201}, and preserving such structures while clustering is paramount. 

Higher-order spectral clustering algorithms account for higher-order interactions, wherein the cut of triangles (capturing triadic interactions) is minimized for network partitioning~\cite{benson2016higher,benson2015tensor,tsourakakis2017scalable}. These methods can be broadly classified into two classes. The first class encodes higher-order structures in an affinity tensor (a higher-order generalization of a similarity matrix), then apply spectral clustering on it \cite{benson2015tensor}. The second class directly constructs the so-called motif Laplacian matrix by counting higher-order structures in the network, then apply spectral clustering on it~\cite{benson2016higher,tsourakakis2017scalable}.
\begin{figure}%
    \centering
    {\includegraphics[width=\columnwidth]{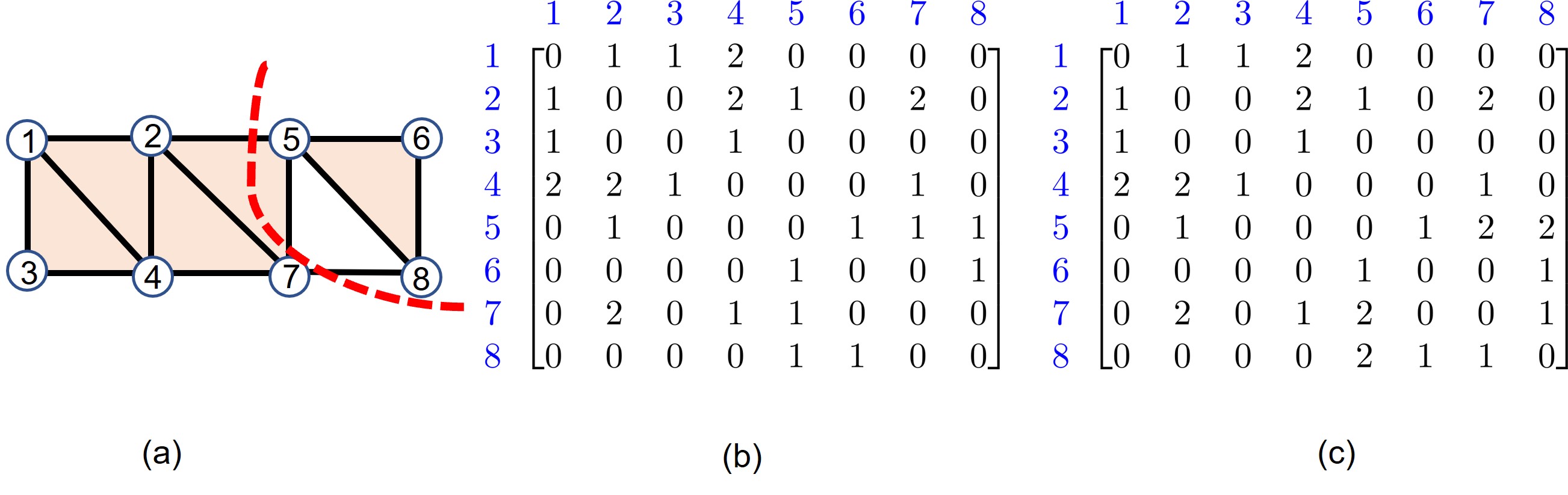} }%
    \caption{(a) Simplicial complex with a 2-simplex cut. (b) Simplicial adjacency matrix. (c) Triangle motif adjacency matrix.}%
    \label{fig:example}
\end{figure}
An affinity tensor or motif adjacency matrix assumes that all pairwise interactions lead to a higher-order (e.g., triadic) relation. That is, it assumes all triangles are filled, while we may have hollow triangles in some cases.
As an example, the network in Fig.~\ref{fig:example}(a) has only pairwise relations between the nodes $\{5,7,8\}$, through $\{5,7\}, \{7,8\}, \{8,5\}$, and the triangle $\{5,7,8\}$ is hollow. This, in other words, means that the network does not contain a triadic interaction between nodes $\{5,7,8\}$. A triangle motif, on the contrary, counts $\{5,7,8\}$ as a triangle and assumes it as equally important as other filled triangles, e.g., $\{1,3,4\}$, $\{2,4,7\}$ and $\{5,6,7\}$ (see its adjacency matrix in Fig.~\ref{fig:example}(b)). 
To distinguish filled higher-order structures, we model networks as simplicial complexes (see its adjacency matrix in Fig.~\ref{fig:example}(c)). Simplicial complexes~\cite{benson2018simplicial} are mathematical objects that model higher-order interactions in networks and are composed of simplices of different orders, such as nodes (0-simplices), edges (1-simplices), and filled triangles (2-simplices) as basic building blocks. 

This work focuses on clustering a network modeled as a simplicial complex using $2$-simplices.
Specifically, we develop a simplicial conductance function, which, when minimized, outputs a partition with a higher density of $2$-simplices within the set and a lower density of 2-simplices across the sets. To circumvent the combinatorial optimization problem when minimizing the simplicial conductance, we propose a new operator called \emph{simplicial adjacency} to encode the similarity between the nodes through the $2$-simplices. The simplicial adjacency matrix can be computed in  closed form as a boundary of a boundary matrix associated with the simplicial complex. We then relate the spectrum of the simplicial Laplacian matrix (derived from the simplicial adjacency matrix) to simplicial conductance through the Cheeger inequality. This allows us to propose a simplicial spectral clustering algorithm. 

\section{Simplicial adjacency matrix}

A simplex is a subset of the vertex set $\cV = \{v_{1},\ldots v_{N}\}$ of a graph with a $k$-simplex (or simplex of order $k$), denoted as $\sigma_{i}^{(k)}= \{v_{i_1},v_{i_2},\ldots v_{i_{k+1}}\}$, being a subset of $\cV$ of cardinality $k+1$. For instance, $\sigma^{(2)}_{i} = \{v_{i_1}, v_{i_2}, v_{i_3}\} \subseteq \cV$ denotes a 2-simplex. An undirected simplicial complex $\cX$ is a finite collection of simplices such as nodes ($0$-simplices), an edge (or $1$-simplices), and a filled-triangle (or $2$-simplices). The order of a simplicial complex is the highest order of the simplices it contains. The boundary matrix $\mB_{k} \in \mathbb{R}^{N_{k-1} \times N_{k}}$ encodes the relation between $(k-1)$-simplices and $k$-simplices, i.e., it encodes which $(k-1)$-simplex is adjacent to which $k$-simplices.  For an undirected simplicial complex $\cX$, the boundary operator $\mB_{k}$ has entries as
$$
[\mB_{k}]_{ij}=\begin{cases}
			1, & \text{if $\sigma^{(k-1)}_{i} \subset \sigma^{(k)}_{j}$},\\
            0, & \text{otherwise.}
		 \end{cases}
$$

Since we do not account for any orientation in $\mB_{k}$, the boundary of a boundary map is not equal to zero, i.e., $\mB_{k}\mB_{k+1} \neq \boldsymbol{0}$. \color{black}

We define the \emph{simplicial adjacency matrix} $\mA_{0,2}$ that encodes the relations between $0$-simplices through $2$-simplices with  its $(i,j)$ entry $[\mA_{0,2}]_{ij}$ equal to the number of $2$-simplices the nodes $v_i$ and $v_j$ appear in, i.e.,
\begin{equation}
    [\mA_{0,2}]_{ij} = \sum_{p=1}^{N_2} \mathbb{I}(\{v_i,v_j\} \in \sigma^{(2)}_p)
    \label{eq:Adjacencymatrix}
\end{equation}
\color{black}
for $i \neq j$ and $[\mA_{0,2}]_{ii} = 0$. Here, $\mathbb{I}(\cdot)$ is the indicator function that returns 1 when its argument is true.  
Let us define the diagonal degree matrix $\mD_{0,2} \in \mathbb{R}^{N_{0} \times N_{0}}$ matrix with entries $[\mD_{0,2}]_{ii} = \sum_{j} [\mA_{0,2}]_{ij}$ and the simplicial Laplacian matrix as $\mL_{0,2} =  \mD_{0,2} - \mA_{0,2}$.
 The normalized simplicial Laplacian is given by $\tilde{\mL}_{0,2} = \mD_{0,2}^{-1/2} \mL_{0,2} \mD_{0,2}^{-1/2}$. The simplicial adjacency matrix in \eqref{eq:Adjacencymatrix} can also be obtained using the boundary operators as follows.

\begin{proposition} The simplicial adjacency matrix $[\mA_{0,2}]_{ij}$ can be computed from the boundary matrices $\mB_1$ and $\mB_2$ as 
\begin{equation}
    [\mA_{0,2}]_{i,j} =  \frac{\left[{\mB}_{12}{\mB}_{12}^{T}\right]_{ij}}{4}.
    \label{eq:simplicialadjacency}
\end{equation}
where ${\mB}_{12}= \mB_{1}\mB_{2} \in \mathbb{R}^{N_{0} \times N_{2}}$.
\end{proposition}

Since each row of $\mB_{1}$ corresponds to a $0$-simplex and has nonzero entries where $1$-simplices are incident on the $0$-simplex and each column of $\mB_{2}$ corresponds to a $2$-simplex and has nonzero entries where $2$-simplices are incident on $1$-simplices. Therefore, the inner product between $i$th row and $j$th column gives a nonzero value if a $0$-simplex is incident upon a $2$-simplex. Hence ${\mB}_{12}$ has nonzero entries only if a $0$-simplex is incident on a $2$-simplex. In other words, $[{\mB}_{12}]_{ij} = 2\mathbb{I}(v_i \in \sigma^{(2)}_j)$. Here, the factor $2$ appears as we consider undirected simplices. Now,  $[\mA_{0,2}]_{i,j}$ is obtained by taking the inner product between the $i$th row and $j$th column of ${\mB}_{12}$ that counts the total number of $2$-simplices the edge between nodes $v_i$ and $v_j$ appear in.
The factor of $4$ is due to undirected simplices.\\

\section{Clustering using $2$-simplices}

This section discusses the proposed algorithm for clustering $0$-simplices (i.e., nodes) in a simplicial complex $\cX$ by leveraging the higher-order relation through $2$-simplices. 

\subsection{Simplicial conductance} 
The $2$-way partitioning of the undirected simplicial complex $\cX$ based on $2$-simplices corresponds to finding a nodal partition such that $0$-simplices within a set have a high density of $2$-simplices, and $0$-simplices across sets have a low density of $2$-simplices. To obtain this, we define the following cut measure:  
\begin{equation}
    {\phi}_{0,2}(\cS)  = \frac{\mathrm{cut}_{0,2}(\cS,\bar{\cS})}{\mathrm{min}\left(\mathrm{vol}_{0,2}(\cS),\mathrm{vol}_{0,2}(\bar{\cS})\right)},
\label{eq:simplicial_Conductance}
\end{equation}
where we call ${\phi}_{0,2}(\cS)$ the \emph{simplicial conductance} induced by 2-simplices. In \eqref{eq:simplicial_Conductance}, $\mathrm{cut}_{0,2}(\cS,\bar{\cS})$ measures the number of $2$-simplex cut of $\cS$, i.e., the number of $2$-simplices that have one vertex in $\cS$ and other vertices in $\bar{\cS}$; $\mathrm{vol}_{0,2}(\cS)$ measures the total number of $2$-simplices having vertices in $\cS$. 
For the optimal partitioning of the simplicial complex, we minimize  $\phi_{0,2}(\cS)$ as 
\begin{equation}
 \phi^{\star}_{0,2}(\cS) =    \underset{\cS\subset \cV}{\mathrm{minimize}} \, \, \phi_{0,2}(\cS).
    \label{eq:simplicial optimization}
\end{equation}
Solving \eqref{eq:simplicial optimization} is NP-hard as it involves evaluating all the possible cuts.
To circumvent the difficulty, we develop an algorithm similar to spectral clustering, namely, simplicial spectral clustering.

\subsection{Cheeger inequality with simplicial adjacency}
We now discuss the relationship between  simplicial conductance to the second smallest eigenvalue of the simplicial Laplacian matrix $\mL_{0,2}$ in the following theorem.
\begin{theorem}
For an undirected simplicial complex $\cX$ having the normalized simplicial Laplacian matrix $\tilde{\mL}_{0,2}$ with $\lambda_2$ being its second smallest eigenvalue, we have 
$$\frac{\lambda_{2}}{2}\leq \phi^{\star}_{0,2}(\cS)\le \sqrt{2\lambda_{2}}.$$
\end{theorem}
\begin{proof} We prove the above Cheeger inequality for an undirected simplicial complex by relating the simplicial conductance $\phi_{0,2}(\cS)$ to the quadratic form of simplicial Laplacian $\mL_{0,2}$. 

Let us define $z_{i}(\cS)$ as number of $2$-simplices with exactly $i$ vertices in $\cS$.  
We can express the $\mathrm{vol}_{0,2}(\cS)$ that measures the total number of $2$-simplices having vertices in $\cS$ as: $$\mathrm{vol}_{0,2}(\cS) = 3z_{3}(\cS)+2z_{2}(\cS)+z_{1}(\cS).$$ 
The factors $3$ and $2$ are due to the undirectedness of the simplicial complex.  Similarly, the cut function  $\mathrm{cut}_{0,2}(\cS,\bar{\cS})$ can be expressed 
% in terms of $z_{i}$ 
as $\mathrm{cut}_{0,2}(\cS,\bar{\cS}) = z_{2}(\cS)+z_1{(\cS)}$. Therefore the simplicial conductance is
\begin{equation}
    \phi_{0,2}(\cS) = \frac{z_{2}(\cS) + z_{1}(\cS)}{\mathrm{min}(\mathrm{vol}_{0,2}(\cS),\mathrm{vol}_{0,2}(\bar{\cS}))}.
\label{eq:simplicial clustering}
\end{equation}

We next express the numerator in \eqref{eq:simplicial clustering} using the quadratic form of the simplicial Laplacian matrix. Let us first define the $3 \times 3$ symmetric matrix $\mM(\sigma_i)$ with the following entries:
$$
[\mM(\sigma^{(2)}_i)]_{mm}=\begin{cases}
			2, & \text{if $v_m \in \sigma^{(2)}_{i}$},\\
            0, & \text{otherwise}
		 \end{cases}
$$
and
$$
[\mM(\sigma_i^{(2)})]_{mn}=\begin{cases}
			-1, & \text{if $\{v_m,v_n\} \in \sigma^{(2)}_i$},\\
            0, & \text{otherwise.}
		 \end{cases}
$$
Let us also define the vector $\ve_\cS \in \{0,1\}^{N_0}$ with entries $[\ve_\cS]_i = \mathbb{I}(v_i \in \cS)$. For a 2-simplex, say $\sigma^{(2)}_i = \{v_{i_1},v_{i_2}, v_{i_3}\}$, we define the $3 \times 1$ vector $\vc(\sigma^{(2)}_i) = [\mathbb{I}(v_{i_1} \in \cS), \mathbb{I}(v_{i_2} \in \cS), \mathbb{I}(v_{i_3} \in \cS)]^T.$ Then the quadratic form can be expressed as

\begin{align}
   \frac{1}{2} \ve_\cS^{T}\mL_{0,2}\ve_\cS &= \sum_{i=1}^{N_2}  \vc^T(\sigma^{(2)}_i)\mM(\sigma^{(2)}_i)\vc(\sigma^{(2)}_i) \nonumber \\
% &= %\sum_{\sigma^{(p)}_{2}(l,m,n)\in \cC_{2}} 
% \sum
% 2\mathbb{I}_l^{2} + 2\mathbb{I}_m^{2}+2\mathbb{I}_n^{2}-2\mathbb{I}_l\mathbb{I}_n-2\mathbb{I}_m\mathbb{I}_n-2\mathbb{I}_l\mathbb{I}_m,\nonumber\\
&=  (z_{2}(\cS) + z_{1}(\cS)) \nonumber \\
& =  \mathrm{cut}_{0,2}(\cS,\bar{\cS}).
\end{align}
 Similarly, we can express $\mathrm{vol}_{0,2}(\cS)$ in the quadratic form of the degree matrix as
\begin{align}
   \frac{1}{2} \ve_\cS^{T}\mD_{0,2}\ve_\cS &= \sum_{i=1}^{N_2}  \vc^T(\sigma^{(2)}_i)\mathrm{diag}\left(\mM(\sigma^{(2)}_i)\right)\vc(\sigma^{(2)}_i) \nonumber \\
% &= %\sum_{\sigma^{(p)}_{2}(l,m,n)\in \cC_{2}} 
% \sum
% 2\mathbb{I}_l^{2} + 2\mathbb{I}_m^{2}+2\mathbb{I}_n^{2}-2\mathbb{I}_l\mathbb{I}_n-2\mathbb{I}_m\mathbb{I}_n-2\mathbb{I}_l\mathbb{I}_m,\nonumber\\
&=  3z_{3}(\cS) + 2z_2(\cS) + z_{1}(\cS) \nonumber \\
& =  \mathrm{vol}_{0,2}(\cS).
\end{align}
\begin{algorithm}[t]
\caption{Clustering $0$-simplices based on $2$-simplices}\label{alg:sp_0}
\begin{algorithmic}[1]
\State {Input:} Simplicial complex $\cX$, boundary matrices 
$\mB_{1}$, $\mB_{2}$
%$\mB_{i}\,\, \forall \,\,i=1,2$
\State {Output:}  Clusters $(\cS,\bar{\cS})$ based on $2$-simplices 
\State Compute $\mA_{0,2}$ from \eqref{eq:simplicialadjacency}, $\mD_{0,2}$, and $\tilde{\mL}_{0,2} = \mD_{0,2}^{-1/2} \mL_{0,2} \mD_{0,2}^{-1/2}$
%\State  Compute the simplicial degree matrix $\mD_{0,2}$. 
% \State  Compute the normalized Laplacian  matrix $\tilde{\mL}_{0,2} = \mD_{0,2}^{-1/2} \mL_{0,2} \mD_{0,2}^{-1/2}$. 
\State $\tilde{\ve}_\cS \gets$  Eigenvector of $\tilde{\mL}_{0,2}$ corresponding to its second smallest eigenvalue
% \State Rank order the entries of $\vx$ and sweep over the ordering of the nodes induced by $\vx$ to find the splitting point.
%\State    \quad \quad\quad    second smallest smallest eigenvalue} 
\State $\gamma_{k} \gets$ Node index corresponding to the $k$th smallest entry of  \text{$\mD^{-1/2}\tilde{\ve}_\cS$}
%\State $.$
%\State  //* Sweep through all the prefixes of $\gamma$   *//
\State $\cS \gets$ $\underset{1\leq k\leq N_0}{\mathrm{min}}$ \, $\phi_{0,2}(\cS_{k})$, where $\cS_k = \{\gamma_1,\gamma_2,\ldots,\gamma_k\}$.
% \If{$|\cS| < |\bar{\cS}|$}
%     \State return $\cS$
% \Else
%     \State return $\bar{\cS}$
% \EndIf
\end{algorithmic}
\end{algorithm}
\color{black}
Hence, \eqref{eq:simplicial optimization} can be equivalently expressed as
\begin{align}
    \underset{\cS}{\mathrm{minimize}} \,\, \frac{\ve_\cS^{T}\mL_{0,2}\ve_\cS}{\ve_\cS^{T}\mD_{0,2}\ve_\cS }\nonumber \\
    \mathrm{s.t.} \quad\boldsymbol{1}^{T}\mD_{0,2}\ve_\cS &= 0,
    \label{eq:Spectral_clustering_simplices}
\end{align}
where $\boldsymbol{1}$ is the all-one vector, the constraint ensures a non-trivial solution.  

By defining $\tilde{\ve}_\cS = \mD_{0,2}^{1/2}\ve_\cS$, \eqref{eq:Spectral_clustering_simplices} can be transformed to 
\begin{align}
    \underset{\cS}{\mathrm{minimize}} \,\, \tilde{\ve}_\cS^{T}\tilde{\mL}_{0,2}\tilde{\ve}_\cS\nonumber \\
    \mathrm{s.to} \quad\boldsymbol{1}^{T}\mD^{1/2}_{0,2}\tilde{\ve}_\cS &= 0, \nonumber \\ \quad \tilde{\ve}_\cS^{T}\tilde{\ve}_\cS=1. 
    \label{eq:Spectral_clustering_simplices}
\end{align}
Thus solving the above problem to minimize the simplicial conductance as in \eqref{eq:simplicial optimization} is analogous to the well-known 2-way spectral clustering problem based on edge cuts and deriving the related Cheeger inequality~\cite{bandeira2013cheeger},\cite{von2007tutorial} remains the same, but with the main difference being that the simplicial conductance is now bounded by the second smallest eigenvalue of normalized simplicial Laplacian. \end{proof}
\noindent

Based on this theorem, we propose the 
simplicial spectral clustering algorithm as detailed in Algorithm~\ref{alg:sp_0}, which extends classical spectral clustering to obtain an optimal 2-way partition of the simplicial complex using $2$-simplices.

\begin{figure*}[t]%
    \centering
    {\includegraphics[width=2\columnwidth]{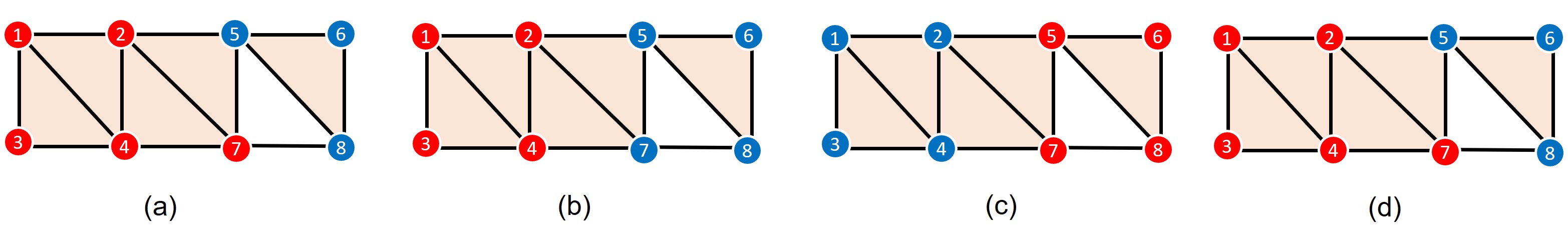} }%
    \caption{Synthetic dataset.
    (a) Ground truth. (b) Communities from graph spectral clustering. (c) Communities from motif spectral clustering. (d) Communities from simplicial spectral clustering.}%
    \label{fig:node partioning}
    \vskip4mm
\end{figure*}

\begin{figure*}%
    \centering
    {\includegraphics[width=2\columnwidth]{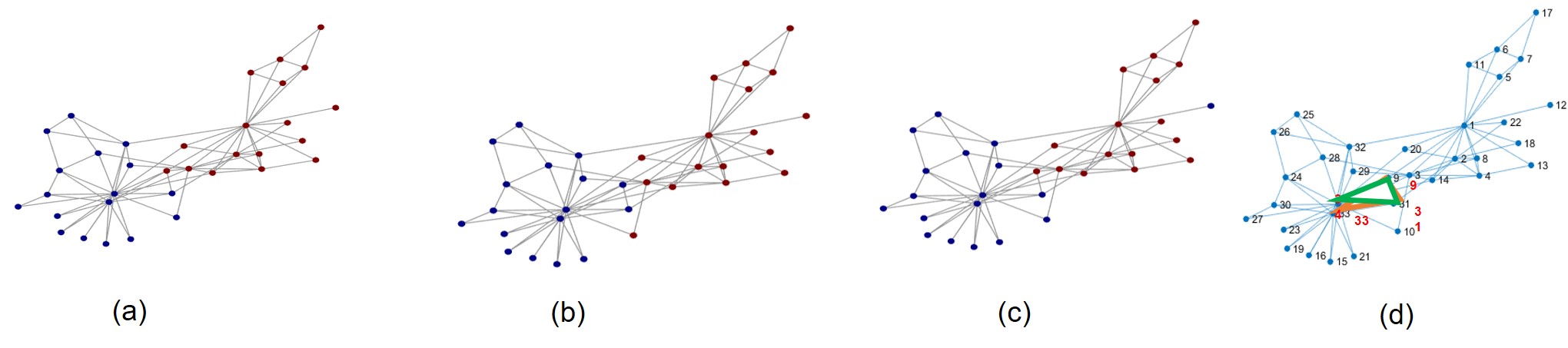} }%
    \caption{Zachary Karate network data. (a) Ground truth network. (b) Communities obtained from motif spectral clustering.   (c) Communities obtained from simplicial spectral clustering. (d) Removed triangles.}%
    \label{fig:zarachy}
\end{figure*}

\begin{figure}%
    \centering
    {\includegraphics[width=\columnwidth]{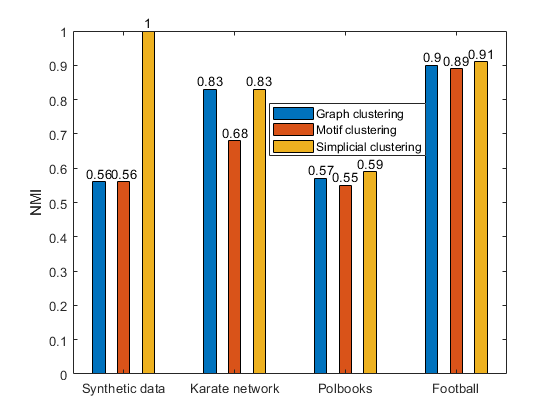} }%
    \caption{Normalized mutual information.}%
    \label{fig:barplot}
\end{figure}
\begin{table}[t]
  \begin{center}
    \caption{DATASETS.}
    \label{tab:table1}
    \begin{tabular}{c|c|c|c} % <-- Alignments: 1st column left, 2nd middle and 3rd right, with vertical lines in between
    \hline
      \textbf{Dataset} & \textbf{\# of nodes} & \textbf{\# of edges} & \textbf{\# of clusters}\\ 
      \hline
      Zachary\cite{zachary1977information} & 34 & 78 & 2 \\
      Polbooks \cite{newman2006modularity}& 105 & 441 & 3\\
      Football \cite{newman2006modularity}& 115 & 613 & 12\\
      \hline
    \end{tabular}
  \end{center}
\end{table}

\section{Numerical experiments}
Numerical experiments to test the proposed method are conducted on synthetic and real datasets.  
We compare the proposed approach with triangle motif-based~\cite{benson2016higher} and graph-based~\cite{von2007tutorial} spectral clustering algorithms. As a performance metric, we use normalized mutual information (NMI)~\cite{NMI}.

\begin{figure*}%
%\begin{subfigure}[b]{0.3\textwidth}
    \centering
    \includegraphics[width=1.75\columnwidth]{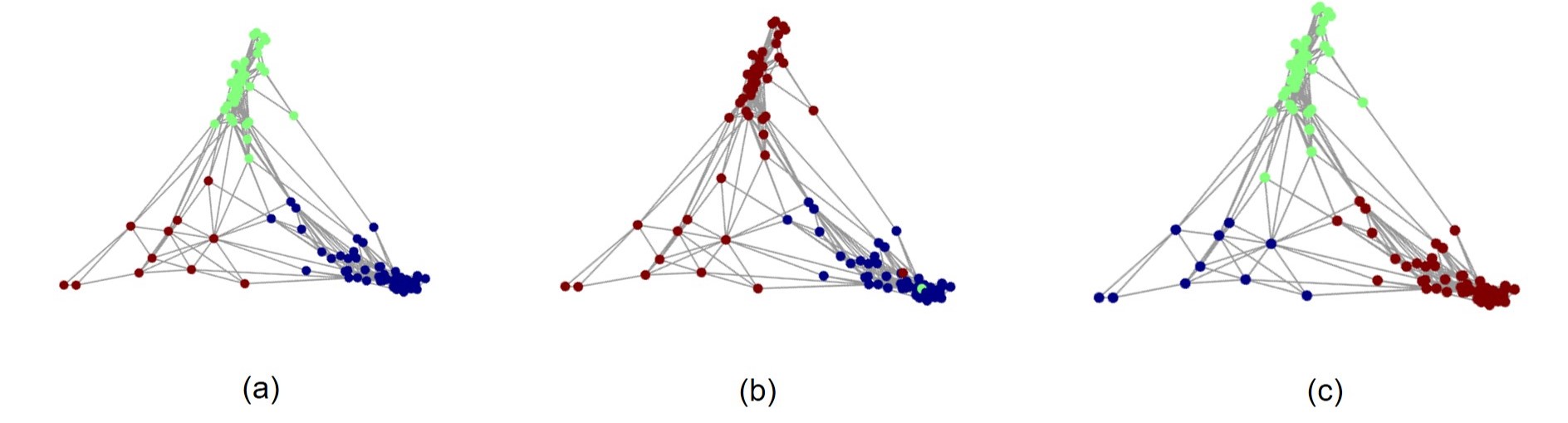} %
%\vspace*{-1.3cm}
   \caption{US Polbooks data. (a) Ground truth network. (b) Communities obtained from motif spectral clustering.   (c) Communities obtained from simplicial spectral clustering.}%
    \label{fig:Polbooks}
%\end{subfigure}
\vskip4mm
\end{figure*}
%\vfill{0.5mm}
\begin{figure*}%
%\begin{subfigure}[b]{0.3\textwidth}
    \centering
    \includegraphics[width=1.6\columnwidth]{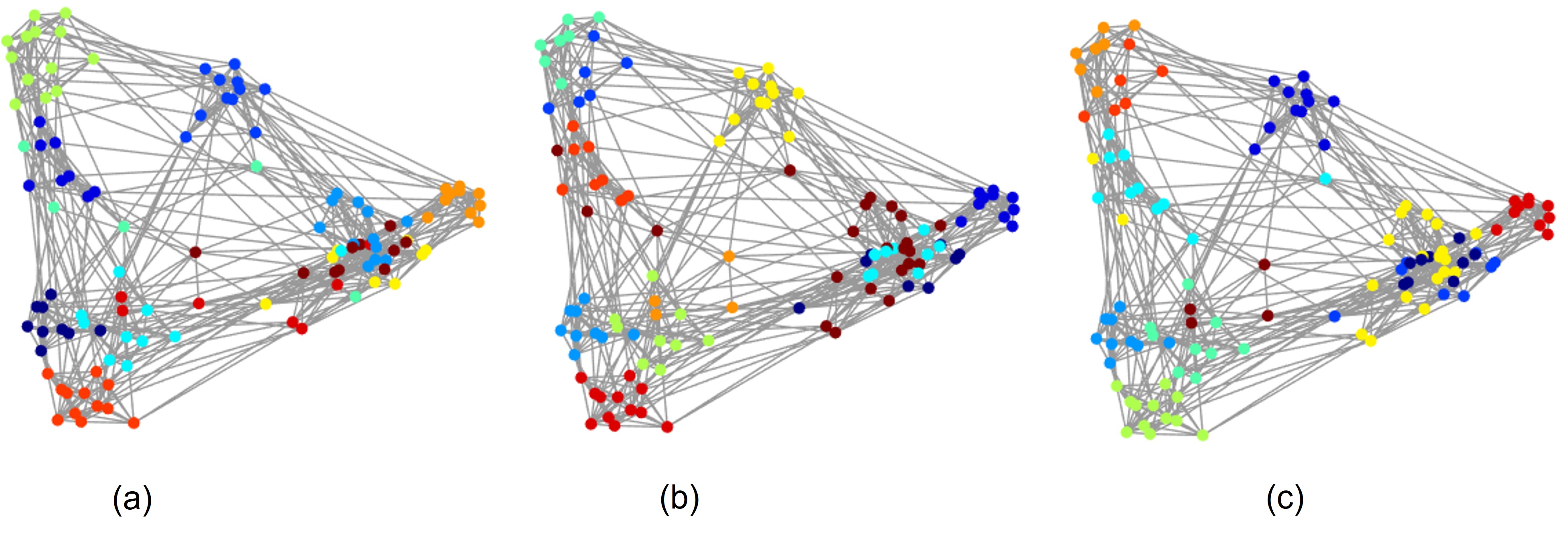} %
%vspace*{-0.3cm}
   \caption{Football network data. (a) Ground truth network. (b) Communities obtained from motif spectral clustering. (c) Communities obtained from simplicial spectral clustering.}%
    \label{fig:football}
%\end{subfigure}
\end{figure*}
\subsection{Synthetic dataset} 

We generate the simplicial complex shown in Fig.~\ref{fig:example}(a), which has $N_{0} = 8$ $0$-simplices, $N_{1} =13$ $1$-simplices, and $N_{2} = 5$, $2$-simplices. So the boundary matrices are $\mB_{1} \in \mathbb{R}^{8 \times 13}$ and $\mB_{2} \in \mathbb{R}^{13 \times 5}$. For this curated simplicial complex, the $0$-simplices are categorized into $2$-classes as shown in Fig.~\ref{fig:node partioning}(a). 
Ground truth class labels for $0$-simplices are based on their relationship with neighboring $0$-simplices through $2$-simplices. 
We compute the simplicial adjacency matrix $\mA_{0,2}\in \mathbb{R}^{8 \times 8}$ using \eqref{eq:simplicialadjacency}, and obtain the clusters using simplicial spectral clustering (Algorithm \ref{alg:sp_0}). 
%The obtained clusters using simplicial spectral approach 
They are shown in Fig. \ref{fig:node partioning}(d).  
Given that it relies on adjacency through $2$-simplices, it
recovers ground truth.
As a comparison, motif-based spectral clustering \cite{benson2016higher} (Fig.~\ref{fig:node partioning}(c)) and graph spectral clustering (Fig.~\ref{fig:node partioning}(b)) do not recover it. For motifs, it assumes that every triangle is filled, while spectral clustering ignores them.
NMI between the obtained clusters and ground truth is reported in Fig.~\ref{fig:barplot}. The proposed method is optimal with NMI = 1.
\subsection{Real datasets} 
%In this section, we discuss the results on real data. In particular, we consider the datasets in the table \ref{tab:table1}. 
%For real datasets, 
We also apply the proposed method on the Zachary Karate Club network~\cite{zachary1977information}, the Polbooks network~\cite{newman2006modularity}, and a football network~\cite{newman2006modularity}; for details about the datasets see Table \ref{tab:table1}. 
Although information about filled triangles is not directly available in these datasets, we report observations about what happens when we assume that some triangles are not filled. 

\subsubsection{Zachary karate club network} This dataset~\cite{zachary1977information} is a well-known community detection dataset about the social relationships of members in a karate club, where the members belong to 2 groups within the club as shown by two different colors in Fig.~\ref{fig:zarachy}(a). 
We follow the procedure from~\cite{benson2016higher} to obtain the triangle motif adjacency matrix. The cluster assignments from motif spectral clustering are shown in Fig.~\ref{fig:zarachy}(b), and NMI is reported in Fig.~\ref{fig:barplot}.  

For the proposed simplicial spectral clustering, we  study the impact of filled triangles with respect to (w.r.t.) the hollow ones. Towards that end, we conduct the following analysis: out of all triangles listed by the triangle adjacency matrix, we remove a few triangles assuming they are open. We remove the triangles formed with edges $\{9,31\}$ and $\{9,34\}$ with members from the 2 ground truth groups. These triangles act as a bridge between the two groups (as shown in Fig.~\ref{fig:zarachy}(d)). The proposed simplicial spectral clustering significantly improves NMI (see Fig.~\ref{fig:barplot}). This asserts our claim that assuming all the open triangles as filled overlooks the importance of filled triangles for graph partitioning.

\subsubsection{US Polbooks data~\cite{newman2006modularity}} %The information about the data is provided in Table \ref{tab:table1}.  
% The analysis is as above, considering all the possible triangles initially as motifs. 
This network has $3$ clusters. Hence we use motif clustering with multiple clusters as a baseline~\cite{benson2016higher}.
For simplicial spectral clustering, the method is:  
$1$) Compute the eigenvectors of the normalized simplical Laplacian matrix; $2$) Collect the eigenvectors corresponding to the $3$ smallest eigenvalues of $\tilde{\mL}$; $3$) Run $K$-means algorithm on obtained eigenvectors with $K=3$ to obtain the clusters. 

The ground truth network results from motif clustering, and simplicial spectral clustering are in Figs. \ref{fig:Polbooks}(a)-(c). As before, we assume a few triangles as hollow. The obtained NMI is reported in Fig.~\ref{fig:barplot}, where the proposed simplicial clustering outperforms the baselines.

\subsubsection{American Football network data~\cite{newman2006modularity}} American football club network is a multicluster dataset with $12$ communities, where we consider a few triangles as hallow. The simplicial spectral clustering for obtaining multicluster assignments is as before.
Figs. \ref{fig:football}(a)-(c) show the ground truth network along with the cluster assignments from motif spectral clustering and simplicial spectral clustering. NMIs for these methods are reported in Fig.~\ref{fig:barplot}, where it can be seen that the proposed method achieves better performance at finding the communities.

\section{Conclusions}
We proposed a simplicial conduction 2-simplicial cut function to incorporate higher-order network interactions. We defined a simplicial Laplacian operator that captures the relationship between the nodes through $2$-simplices and developed a Cheeger inequality relating the second smallest eigenvalue of the proposed simplicial Laplacian matrix to the optimal simplicial conductance. Further, leveraging the Cheeger inequality, we developed a new simplicial spectral clustering algorithm, which was found to cluster networks better than edge cut-based and triangle motif-based  spectral clustering methods while being able to distinguish filled and hollow triangles.

% \pagebreak

\bibliographystyle{IEEEtran}
\bibliography{refs.bib}

% Generated by IEEEtran.bst, version: 1.14 (2015/08/26)
\begin{thebibliography}{10}
\providecommand{\url}[1]{#1}
\csname url@samestyle\endcsname
\providecommand{\newblock}{\relax}
\providecommand{\bibinfo}[2]{#2}
\providecommand{\BIBentrySTDinterwordspacing}{\spaceskip=0pt\relax}
\providecommand{\BIBentryALTinterwordstretchfactor}{4}
\providecommand{\BIBentryALTinterwordspacing}{\spaceskip=\fontdimen2\font plus
\BIBentryALTinterwordstretchfactor\fontdimen3\font minus
  \fontdimen4\font\relax}
\providecommand{\BIBforeignlanguage}[2]{{%
\expandafter\ifx\csname l@#1\endcsname\relax
\typeout{** WARNING: IEEEtran.bst: No hyphenation pattern has been}%
\typeout{** loaded for the language `#1'. Using the pattern for}%
\typeout{** the default language instead.}%
\else
\language=\csname l@#1\endcsname
\fi
#2}}
\providecommand{\BIBdecl}{\relax}
\BIBdecl

\bibitem{Fortunato2010}
S.~Fortunato, ``Community detection in graphs,'' \emph{Physics {Reports}}, vol.
  486, no.~3, pp. 75--174, 2010.

\bibitem{benson2016higher}
A.~R. Benson, D.~F. Gleich, and J.~Leskovec, ``Higher-order organization of
  complex networks,'' \emph{Science}, vol. 353, no. 6295, pp. 163--166, 2016.

\bibitem{newman2006modularity}
M.~E. Newman, ``Modularity and community structure in networks,''
  \emph{Proceedings of the national academy of sciences}, vol. 103, no.~23, pp.
  8577--8582, 2006.

\bibitem{lu2018community}
Z.~Lu, J.~Wahlstr{\"o}m, and A.~Nehorai, ``Community detection in complex
  networks via clique conductance,'' \emph{Scientific reports}, vol.~8, no.~1,
  p. 5982, 2018.

\bibitem{shi2000normalized}
J.~Shi and J.~Malik, ``Normalized cuts and image segmentation,'' \emph{IEEE
  Transactions on pattern analysis and machine intelligence}, vol.~22, no.~8,
  pp. 888--905, 2000.

\bibitem{ng2001spectral}
A.~Ng, M.~Jordan, and Y.~Weiss, ``On spectral clustering: Analysis and an
  algorithm,'' \emph{Advances in neural information processing systems},
  vol.~14, 2001.

\bibitem{von2007tutorial}
U.~Von~Luxburg, ``A tutorial on spectral clustering,'' \emph{Statistics and
  computing}, vol.~17, pp. 395--416, 2007.

\bibitem{newman2001structure}
M.~E. Newman, ``The structure of scientific collaboration networks,''
  \emph{Proceedings of the national academy of sciences}, vol.~98, no.~2, pp.
  404--409, 2001.

\bibitem{granovetter1973strength}
M.~S. Granovetter, ``The strength of weak ties,'' \emph{American journal of
  sociology}, vol.~78, no.~6, pp. 1360--1380, 1973.

\bibitem{BATTISTON20201}
F.~Battiston, G.~Cencetti, I.~Iacopini, V.~Latora, M.~Lucas, A.~Patania, J.-G.
  Young, and G.~Petri, ``Networks beyond pairwise interactions: Structure and
  dynamics,'' \emph{Physics Reports}, vol. 874, pp. 1--92, 2020.

\bibitem{benson2015tensor}
A.~R. Benson, D.~F. Gleich, and J.~Leskovec, ``Tensor spectral clustering for
  partitioning higher-order network structures,'' in \emph{Proceedings of the
  2015 SIAM International Conference on Data Mining}.\hskip 1em plus 0.5em
  minus 0.4em\relax SIAM, 2015, pp. 118--126.

\bibitem{tsourakakis2017scalable}
C.~E. Tsourakakis, J.~Pachocki, and M.~Mitzenmacher, ``Scalable motif-aware
  graph clustering,'' in \emph{Proceedings of the 26th International Conference
  on World Wide Web}, 2017, pp. 1451--1460.

\bibitem{benson2018simplicial}
A.~R. Benson, R.~Abebe, M.~T. Schaub, A.~Jadbabaie, and J.~Kleinberg,
  ``Simplicial closure and higher-order link prediction,'' \emph{Proceedings of
  the National Academy of Sciences}, vol. 115, no.~48, pp. E11\,221--E11\,230,
  2018.

\bibitem{bandeira2013cheeger}
A.~S. Bandeira, A.~Singer, and D.~A. Spielman, ``A cheeger inequality for the
  graph connection laplacian,'' \emph{SIAM Journal on Matrix Analysis and
  Applications}, vol.~34, no.~4, pp. 1611--1630, 2013.

\bibitem{zachary1977information}
W.~W. Zachary, ``An information flow model for conflict and fission in small
  groups,'' \emph{Journal of anthropological research}, vol.~33, no.~4, pp.
  452--473, 1977.

\bibitem{NMI}
L.~Ana and A.~Jain, ``Robust data clustering,'' in \emph{2003 IEEE Computer
  Society Conference on Computer Vision and Pattern Recognition, 2003.
  Proceedings.}, vol.~2, 2003, pp. II--II.

\end{thebibliography}

\end{document}